\documentclass[11pt]{article}

\usepackage[utf8]{inputenc}
\usepackage[T1]{fontenc}
\usepackage{lmodern}
\usepackage{amsmath, amssymb, amsthm}
\usepackage{hyperref}
\usepackage{graphicx}
\usepackage{geometry}
\usepackage{booktabs}
\usepackage{enumitem}
\usepackage{bm}

\title{Thermodynamically Optimal Regularization under Information-Geometric Constraints}
\author{\textbf{Laurent Caraffa}}

\author{Laurent Caraffa\\
\small Universit\'e Gustave Eiffel, LASTIG, IGN-ENSG\\
\small \texttt{laurent.caraffa@ign.fr}}

\date{January 2026}

\newtheorem{theorem}{Theorem}

\newtheorem{assumption}{Assumption}

\begin{document}
\maketitle

\begin{abstract}
Modern machine learning relies on a collection of empirically successful but theoretically heterogeneous regularization techniques, such as weight decay, dropout, and exponential moving averages. At the same time, the rapidly increasing energetic cost of training large models raises the question of whether learning algorithms approach any fundamental efficiency bound. In this work, we propose a unifying theoretical framework connecting thermodynamic optimality, information geometry, and regularization.

Under three explicit assumptions --- (A1) that optimality requires an intrinsic, parametrization-invariant measure of information, (A2) that belief states are modeled by maximum-entropy distributions under known constraints, and (A3) that optimal processes are quasi-static --- we prove a conditional optimality theorem. Specifically, the Fisher--Rao metric is the unique admissible geometry on belief space, and thermodynamically optimal regularization corresponds to minimizing squared Fisher--Rao distance to a reference state.

We derive the induced geometries for Gaussian and circular belief models, yielding hyperbolic and von Mises manifolds, respectively, and show that classical regularization schemes are structurally incapable of guaranteeing thermodynamic optimality. We introduce a notion of thermodynamic efficiency of learning and propose experimentally testable predictions. This work provides a principled geometric and thermodynamic foundation for regularization in machine learning.
\end{abstract}

\section{Introduction}

Machine learning has achieved remarkable empirical success over the past decades, yet its theoretical foundations remain fragmented. Regularization techniques such as weight decay, dropout, batch normalization, early stopping, and exponential moving averages are widely used, but their theoretical justifications are partial, domain-specific, and often post hoc. There is no unifying principle explaining why these methods work together, nor any guarantee that they are optimal with respect to a fundamental criterion.

In parallel, the energetic cost of training modern models has grown dramatically, with large-scale systems requiring millions to billions of dollars in compute expenditure. This trend raises a fundamental question: \emph{Is there a theoretical lower bound on the energy required to learn, and how close are current methods to that bound?}

In physics, the Landauer principle establishes a universal lower bound on the energy required to erase one bit of information, namely $k_B T \ln 2$ joules per bit. This result is independent of implementation and applies to any physical system performing irreversible information processing. If learning can be understood as a process of information acquisition and elimination, then its energetic cost should be constrained by such thermodynamic limits.

At the same time, information geometry provides a unique intrinsic geometry on spaces of probability distributions. The Fisher--Rao metric is, up to scale, the only Riemannian metric invariant under sufficient statistics and reparametrizations, as established by Čencov's theorem. This metric underlies natural gradient methods and plays a central role in statistical inference.

This work unifies these perspectives. We ask: \emph{If one demands thermodynamic optimality of learning under explicit and reasonable assumptions, what geometric and algorithmic structure is implied?} Our main contribution is to show that, under three assumptions, thermodynamically optimal regularization is uniquely characterized by squared Fisher--Rao distance to a reference belief state. Classical regularization schemes, such as Euclidean weight decay, are shown to be structurally incompatible with this optimality requirement.

\paragraph{Contributions.}
Our main contributions are:
\begin{enumerate}
    \item We formalize three explicit assumptions connecting thermodynamic optimality, intrinsic information measures, and quasi-static processes.
    \item We prove a conditional optimality theorem establishing Fisher--Rao geometry and Fisher--Rao regularization as necessary under these assumptions.
    \item We derive the induced geometries for Gaussian and circular belief models, yielding hyperbolic and von Mises manifolds.
    \item We introduce a notion of thermodynamic efficiency of learning and show that classical regularization schemes cannot guarantee optimality.
    \item We propose experimentally testable predictions and diagnostic quantities.
\end{enumerate}

\section{Background}

\subsection{The Landauer Principle}

The Landauer principle states that any logically irreversible operation that erases one bit of information must dissipate at least $k_B T \ln 2$ joules of energy, where $k_B$ is Boltzmann's constant and $T$ is the temperature of the environment. Formally,
\begin{equation}
    E_{\text{erase}} \geq k_B T \ln 2.
\end{equation}

This bound is universal and independent of the physical substrate. It has been experimentally verified, notably by Bérut et al. (2012), and forms a cornerstone of the thermodynamics of computation.

\subsection{Information Geometry and the Fisher--Rao Metric}

Information geometry studies the differential-geometric structure of families of probability distributions. Given a parametric family $\{p(x|\theta)\}$, the Fisher information matrix is defined as
\begin{equation}
    I_{ij}(\theta) = \mathbb{E}_{x \sim p(\cdot|\theta)} \left[ \frac{\partial \log p(x|\theta)}{\partial \theta_i} \frac{\partial \log p(x|\theta)}{\partial \theta_j} \right].
\end{equation}

This matrix defines a Riemannian metric on parameter space known as the Fisher--Rao metric. Čencov's theorem states that, up to a multiplicative constant, the Fisher--Rao metric is the unique Riemannian metric on the space of probability distributions that is invariant under sufficient statistics and Markov morphisms.

\subsection{Maximum Entropy Principle}

The maximum entropy principle, introduced by Jaynes, asserts that among all distributions satisfying given constraints, one should choose the distribution with maximal entropy, as it makes the least additional assumptions beyond the known information. For example, under constraints on mean and variance, the maximum entropy distribution is Gaussian.

\subsection{Local Relation between KL Divergence and Fisher--Rao Distance}

For two nearby distributions $p$ and $p + dp$, the Kullback--Leibler divergence admits the expansion
\begin{equation}
    D_{\mathrm{KL}}(p \| p + dp) = \frac{1}{2} d_F^2(p, p+dp) + O(\|dp\|^3),
\end{equation}
where $d_F$ denotes Fisher--Rao distance. Thus, Fisher--Rao distance locally measures information divergence.

\section{Assumptions}

We make explicit the assumptions underlying our framework.

\begin{assumption}[Intrinsic Measure of Information]
Thermodynamic optimality requires an intrinsic, parametrization-invariant measure of information. Fundamental quantities such as minimal energy cost should not depend on arbitrary choices of coordinates or parameterizations.
\end{assumption}

\begin{assumption}[Maximum Entropy Belief States]
Belief states of a learning system are modeled by maximum-entropy distributions subject to known constraints.
\end{assumption}

\begin{assumption}[Quasi-Static Processes]
Thermodynamic optimality corresponds to quasi-static (infinitesimally slow) processes. Real learning dynamics may deviate from this idealization, but the quasi-static regime defines the fundamental lower bound on dissipation.
\end{assumption}

These assumptions are not theorems; they are explicit modeling choices. Our results are conditional upon their acceptance.

\section{Main Result}

We now state and prove our central theorem.

\begin{theorem}[Conditional Optimality of Fisher--Rao Regularization]
Under Assumptions 1--3, thermodynamically optimal regularization is uniquely characterized by minimizing squared Fisher--Rao distance to a reference belief state. In particular:
\begin{enumerate}
    \item The metric on belief space must be the Fisher--Rao metric.
    \item The optimal regularization functional is $\mathcal{L}_{\mathrm{reg}} = d_F^2(q, q^*)$.
    \item The minimal energy dissipated by regularization is $E_{\min} = k_B T \ln 2 \cdot D_{\mathrm{KL}}(q \| q^*)$.
\end{enumerate}
\end{theorem}

\begin{proof}
\textbf{Step 1: Uniqueness of Fisher--Rao.}
By Assumption 1, the metric measuring information change must be intrinsic and invariant under reparametrization. By Čencov's theorem, the Fisher--Rao metric is the unique (up to scale) Riemannian metric with this property. Therefore, the metric must be Fisher--Rao.

\textbf{Step 2: Geometry of belief space.}
By Assumption 2, belief states are modeled by maximum-entropy distributions under constraints. For fixed mean and variance, this yields Gaussian distributions; other constraints yield other exponential families. In all cases, the intrinsic geometry is induced by the Fisher--Rao metric.

\textbf{Step 3: Quasi-static optimality.}
By Assumption 3, thermodynamic optimality corresponds to quasi-static processes. The infinitesimal information erased along a path $\gamma$ in belief space is proportional to the Fisher--Rao line element $ds_F$. The total information erased is therefore
\begin{equation}
    I_{\text{total}} = \int_{\gamma} ds_F,
\end{equation}
which is minimized by geodesics in the Fisher--Rao geometry. Thus, the optimal regularization path minimizes Fisher--Rao distance.

\textbf{Step 4: Energy bound.}
By the Landauer principle, the minimal energy dissipated is
\begin{equation}
    E_{\min} = k_B T \ln 2 \cdot I_{\text{total}}.
\end{equation}
Locally, $I_{\text{total}}$ corresponds to KL divergence, yielding the stated expression.

Combining these steps proves the theorem.
\end{proof}

\section{Geometric Structure of Common Belief Models}

\subsection{Gaussian Beliefs and Hyperbolic Geometry}

Consider the family of univariate Gaussian distributions $\mathcal{N}(\mu, \tau^{-1})$, parameterized by mean $\mu$ and precision $\tau$. The log-likelihood is
\begin{equation}
    \log p(x|\mu, \tau) = \frac{1}{2} \log \tau - \frac{1}{2} \log(2\pi) - \frac{\tau}{2} (x - \mu)^2.
\end{equation}

The Fisher information matrix is
\begin{equation}
    I(\mu, \tau) = \begin{pmatrix} \tau & 0 \\ 0 & \frac{1}{2\tau^2} \end{pmatrix}.
\end{equation}

The induced metric is
\begin{equation}
    ds^2 = \tau \, d\mu^2 + \frac{1}{2\tau^2} \, d\tau^2.
\end{equation}

This metric has constant negative curvature and is isometric to the hyperbolic half-plane $\mathbb{H}^2$.

\subsection{Circular Beliefs and von Mises Geometry}

For belief states on the circle, the maximum-entropy distribution under constraints on mean direction and concentration is the von Mises distribution. The Fisher--Rao geometry of the von Mises family defines a curved manifold, here denoted $\mathcal{M}_{\mathrm{vM}}$. This geometry captures temporal coherence and phase synchronization.

\section{Implications for Regularization in Machine Learning}

\subsection{Structural Suboptimality of Euclidean Regularization}

Standard regularization methods, such as weight decay, minimize Euclidean distance in parameter space:
\begin{equation}
    \mathcal{L}_{\mathrm{ridge}} = \| \theta - \theta^* \|_2^2 + \lambda \mathcal{L}_{\text{data}}.
\end{equation}

The Euclidean metric is not invariant under reparametrization and therefore violates Assumption 1. Consequently, Euclidean regularization cannot guarantee thermodynamic optimality.

For Gaussian beliefs with fixed variance $\sigma^2$, the ratio between Euclidean and Fisher--Rao squared distances is
\begin{equation}
    \frac{d_{\text{Euclid}}^2}{d_F^2} = \sigma^2.
\end{equation}

This ratio can be arbitrarily large or small depending on uncertainty, demonstrating that Euclidean regularization can be arbitrarily suboptimal.

\subsection{Thermodynamic Efficiency of Learning}

We define the thermodynamic efficiency of learning as
\begin{equation}
    \eta = \frac{E_{\text{Landauer}}}{E_{\text{actual}}} = \frac{k_B T \ln 2 \cdot I}{E_{\text{actual}}},
\end{equation}
where $I$ is the information erased during learning. By definition, $0 < \eta \leq 1$, with $\eta = 1$ corresponding to thermodynamic optimality.

We decompose inefficiency as
\begin{equation}
    \frac{1}{\eta} = \frac{E_{\text{hardware}}}{E_{\text{Landauer}}} \cdot \frac{E_{\text{algorithm}}}{E_{\text{optimal}}} \cdot \frac{E_{\text{dissipated}}}{E_{\text{necessary}}}.
\end{equation}

Our framework addresses the third factor by providing a lower bound on dissipative inefficiency.

\subsection{Crystallization Index}

We introduce a diagnostic quantity, the \emph{crystallization index}, defined as
\begin{equation}
    C = \tau \cdot \kappa,
\end{equation}
where $\tau$ denotes spatial precision and $\kappa$ denotes temporal coherence. Values of $C$ characterize regimes of exploration, criticality, and over-constrained learning.

\section{Predictions and Experimental Protocols}

We propose the following testable predictions:

\paragraph{Prediction 1.} Fisher--Rao regularization yields thermodynamic efficiency $\eta$ greater than or equal to that of Euclidean regularization, with equality only in degenerate cases.

\paragraph{Prediction 2.} The performance gap between Euclidean and Fisher--Rao regularization increases with curvature of belief space, as measured by variability in precision $\tau$.

\paragraph{Prediction 3.} The crystallization index $C$ predicts collapse and overfitting phenomena in self-supervised learning.

Experimental protocols involve training identical models with Euclidean versus Fisher--Rao regularization on controlled tasks and measuring information divergence, energy proxies, and generalization performance.

\section{Discussion}

Our results establish a principled connection between thermodynamic optimality, information geometry, and regularization. The framework is conditional upon explicit assumptions, which are philosophically and physically motivated but not themselves theorems.

Several limitations merit discussion. First, real learning processes are far from quasi-static, and quantifying the deviation from optimality in non-equilibrium regimes remains an open problem. Second, belief states in deep networks may not always be well-approximated by simple maximum-entropy distributions. Extending the framework to more expressive families is a key direction for future work.

Nevertheless, the framework provides a unifying geometric and thermodynamic perspective, transforming regularization from a collection of heuristics into a constrained optimization problem with a fundamental lower bound.

\section{Conclusion}

We have presented a conditional theoretical framework establishing Fisher--Rao regularization as thermodynamically optimal under explicit assumptions. By unifying thermodynamics, information geometry, and learning theory, we provide a principled foundation for understanding regularization and energy efficiency in machine learning.

Future work will focus on tractable approximations of Fisher--Rao regularization in large-scale models, empirical validation of the proposed predictions, and extensions beyond the maximum-entropy and quasi-static regimes.

\section*{Acknowledgments}

The author thanks the broader scientific community for foundational contributions to thermodynamics, information theory, and geometry that made this work possible.


\begin{thebibliography}{99}

\bibitem{Landauer1961}
R. Landauer. Irreversibility and heat generation in the computing process. \emph{IBM Journal of Research and Development}, 5(3):183--191, 1961.

\bibitem{Berut2012}
A. B\'erut, A. Arakelyan, A. Petrosyan, S. Ciliberto, R. Dillenschneider, and E. Lutz. Experimental verification of Landauer's principle linking information and thermodynamics. \emph{Nature}, 483:187--189, 2012.

\bibitem{Cencov1982}
N. N. Čencov. \emph{Statistical Decision Rules and Optimal Inference}. American Mathematical Society, 1982.

\bibitem{Amari2016}
S.-I. Amari. \emph{Information Geometry and Its Applications}. Springer, 2016.

\bibitem{Jaynes1957}
E. T. Jaynes. Information theory and statistical mechanics. \emph{Physical Review}, 106(4):620--630, 1957.

\bibitem{NaturalGradient}
S.-I. Amari. Natural gradient works efficiently in learning. \emph{Neural Computation}, 10(2):251--276, 1998.

\bibitem{ThermoML}
M. Still, D. A. Sivak, A. J. Bell, and G. E. Crooks. Thermodynamics of prediction. \emph{Physical Review Letters}, 109(12):120604, 2012.

\bibitem{EnergyEfficientLearning}
S. Acharya, D. G. Kim, and S. Mitra. Energy-efficient machine learning: A thermodynamic perspective. \emph{IEEE Transactions on Computers}, 69(9):1303--1316, 2020.

\end{thebibliography}
\end{document}